\documentclass{amsart}
\usepackage{algpseudocode}
\usepackage{algorithmicx}
\usepackage{algorithm}
\usepackage{amsfonts,bm}
\usepackage{graphicx}

\algdef{SE}[DOWHILE]{Do}{DoWhile}{\algorithmicdo}[1]{\algorithmicwhile\ #1}%
\newcommand{\R}{\mathbb{R}}
\let\vec\mathbf
\newcommand{\Cref}[1]{Figure~\ref{#1}}
\newtheorem{theorem}{Theorem}
\newtheorem{definition}{Definition}
\newtheorem{proposition}{Proposition}

\begin{document}

\title{DCSVM: Fast Multi-class Classification using Support Vector Machines}
\author{Duleep Rathgamage Don}
\address{Department
of Mathematical Sciences\\Georgia Southern University\\Statesboro,
GA 30460}
\email{dr04108@georgiasouthern.edu}
\author{Ionut E. Iacob}
\address{Department
of Mathematical Sciences\\Georgia Southern University\\Statesboro,
GA 30460}
\email{ieiacob@georgiasouthern.edu}


\maketitle

% REQUIRED
\begin{abstract}
  We present DCSVM, an efficient algorithm for multi-class classification using Support Vector Machines. DCSVM is a divide and conquer algorithm which relies on data sparsity in high dimensional space and performs a smart partitioning of the whole training data set into disjoint subsets that are easily separable. A single prediction performed between two partitions eliminates at once one or more classes in one partition, leaving only a reduced number of candidate classes for subsequent steps. The algorithm continues recursively, reducing the number of classes at each step, until a final binary decision is made between the last two classes left in the competition. In the best case scenario, our algorithm makes a final decision between $k$ classes in $O(\log k)$ decision steps and in the worst case scenario DCSVM makes a final decision in $k-1$ steps, which is not worse than the existent techniques.
\end{abstract}

% REQUIRED
%\begin{keywords}
%multiclass classification, SVM, divide and conquer.
%\end{keywords}

% REQUIRED
%\begin{AMS}
%  62H30, 68T10
%\end{AMS}

\section{Introduction}
The curse of dimensionality refers to various phenomena that arise when analyzing and organizing data in high-dimensional spaces (often with hundreds or thousands of dimensions) that do not occur in low-dimensional settings such as the three-dimensional physical space of everyday experience. The expression was coined by Richard E. Bellman in a highly acclaimed article considering problems in dynamic optimization \cite{bellman:57:dynamic,bellman:03:dynamic}. In essence, as dimensionality increases, the volume of the space increases rapidly, and the available data become sparser and sparser. In general, this sparsity is problematic for any method that requires statistical significance. In order to obtain a statistically sound and reliable result, the amount of data needed to support the result often grows exponentially with the dimensionality, which would prevent common data processing techniques from being efficient.
\begin{figure}[ht]
\centering
\includegraphics[scale=0.65]{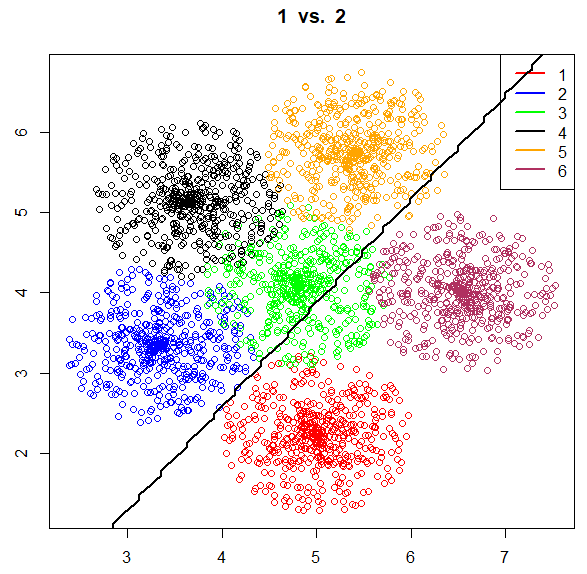}
\caption{Binary SVM classifier for classes 1 and 2 out of a dataset of six classes}\label{fig:artificial1}
\end{figure}

Since its introduction, the Support Vector Machines (SVM) \cite{cortes:95:support-vectornetworks} has quickly become a popular tool for classification which has attracted a lot of interest in the machine learning community. However, SVM is primarily a binary classification tool. The multiclass classification with SVM is still an ongoing research problem (see, for example, \cite{bishwas:17:allpair, xu:17:multi, perez:18:mc2esvm, palacios:18:multiclass} for some recent work). We present an SVM-based multi-class classification method that exploits the curse of dimensionality to efficiently perform classification of highly dimensional data.

The Divide and Conquer SVM (DCSVM) algorithm's idea is based on the following simple observation, best described using the example in \Cref{fig:artificial1}. The figure shows 6 classes (1 -- red, 2 -- blue, 3 -- green, 4 -- black, 5 -- orange, and 6 -- maroon) of two-dimensional points and a linear SVM separation of classes 1 and 2 (the line that separates the points in these classes). It happens that the SVM model for classifying classes 1 and 2 completely separates the points in classes 4 (which takes class 2's side) and 6 (which takes class 1's side). Moreover, the classifier does a relatively good job classifying most points of the class 5 as class 2 (with relatively few points classified as 1) and a poor job on classifying the points of class 3 (as the points in this class are classified about half as 1's and the other half as 2's). With DCSVM we use the SVM classifier for classes 1 and 2 for a candidate of an unknown class: if the classifier predicts 1, then we next decide between classes 1, 6, 3, and 5; if the classifier predicts class 2, then we next decide between classes 2, 4, 3, and 5. Notice that in either case one or more classes are eliminated, and we are left to predict a fewer number of classes. That is, a multi-class classification problem of a smaller size (less classes). The algorithm then proceeds recursively on the smaller problem. In the best case scenario at each step half of the $k$ classes will be eliminated and the algorithm will finish in $\lceil\log k\rceil$ steps. Notice that, in the above scenario, classes 2 and 4 are completely separated from classes 1 and 6, whereas classes 3 and 5 are not clearly on one side or the other of the separation line. For this reason, classes 3 and 5 are part of the next decision step, regardless of the prediction of the first classifier.

However, there is a significant difference between classes 3 and 5. While class 3 is almost divided in half by the separation line, class 5 can be predicted as ``2" with a relatively small error. In DCSVM we use a threshold value $\theta$ to indicate the maximum classification error accepted in order to consider a class on one side or the other of a separation line. For instance, let us consider that only 2\% of the points of class 5 are on the same side as class 1. With the threshold value set to 0.02, DCSVM will separate classes 1, 3, and 6 (when 1 is predicted) from classes 2, 3, 4, and 5 (when 2 is predicted). A higher threshold value will produce a better separation of classes (less overlapping) and less classes to process in the subsequent steps. This comes at the price of possibly sacrificing the accuracy of the final prediction.

Clearly, the method presented in the example above is suitable for multiclass classification using a binary classifier, in general. Our choice of using SVM is based on the SVM algorithm's remarkable power in producing accurate binary classification.

The content of this article is organized as follows. We give a brief description of binary classification with SVM and related work on using SVM for multi-class classification in Section~\ref{sec:preliminaries}. DCSVM is described in detail in Section~\ref{sec:dcsvm} and experimental results (including performance comparisons with one-versus-one approach) are given in Section~\ref{sec:experimental}. We conclude in Section~\ref{sec:conclusion}.

\section{Preliminaries and related work}\label{sec:preliminaries}

Support Vector Machines (SVM) \cite{cortes:95:support-vectornetworks} was primary developed as a tool for the binary classification problem by finding a separation hyperplane for the classes in feature space. If such a plane cannot be find, the ``separating plane" requirement is softened and a maximal margin separation is produced instead. Formally, the problem of finding a maximal margin separation can be stated as a quadratic optimization problem. Given a set of $l$ training vectors $\vec{x}_i\in\R^d$, $i = 1\dots l$ with labels $y_i\in\{-1,1\}$ and a feature space projection $\Phi : \R^d\rightarrow H$, the SVM method consists in finding the solution of the following:
\begin{eqnarray*}
  \min_{\vec{w}\in\R^d, b\in\R,\xi_i\in\R} &\ \ \ \ & \frac{1}{2}\vec{w}^T\vec{w} + C\sum_{i=1}^l\xi_i \\
  \mbox{subject to } & & y_i\left(\left\langle\vec{w},\Phi(\vec{x}_i)\right\rangle + b\right) \ge 1 - \xi_i \\
  & & \xi_i \ge 0,\ i=1\dots l
\end{eqnarray*}
where $\vec{w}\in\R^d$ is the weights vector and $C\in\R^+$ is a cost regularization constant. The corresponding decision function is:
$$f(\vec{x}) = sign\left(\left\langle\vec{w},\Phi(\vec{x})\right\rangle + b\right)$$
An interesting property of the method is that the dot product can be represented by a kernel function:
$$k(\vec{x},\vec{x}') = \left\langle\Phi(\vec{x}),\Phi(\vec{x}')\right\rangle$$
which is computationally much less expensive than actually projecting $\vec{x}$ and $\vec{x}'$ into the feature space $H$.

In the case of multiple classes, the problem formulation becomes more complicated and inherently more difficult to address. Given a set of $l$ training vectors $\vec{x}_i\in\R^d$, $i = 1\dots l$ with labels $y_i\in\{1, \dots, k\}$, one must find a way to distinguish between $k$ classes.

Several approaches were proposed, which can be grouped into direct methods (a single optimization problem formulation for multi-class classification) and indirect methods (using multiple binary classifiers to produce multi-class classification). Many of the indirect methods were introduced, in fact, as methods for multi-class classification using binary classifiers, in general. They are not limited to the SVM method.

A comparison \cite{hsu:02:comparison} of these methods of multi-class classification using binary SVM classifiers shows that one-versus-one method and its DAG improvement are more suitable for practical use.

\subsection{Direct formulation of multi-class classification}
Direct formulations to distinguish between $k$ classes in a single optimization problem were given in \cite{vapnik:98:statistical, bredensteiner:99:multicategory, weston:99:multi, crammer:02:algorithmic} or, more recently, in \cite{he:12:simplified, xu:17:multi}. Each of these formulations has a single objective function for training all $k$-binary SVMs simultaneously and maximize the margins from each class to the remaining ones. The decision function then chooses the ``best classified" class.

For instance, Crammer et al. in \cite{crammer:02:algorithmic} solve the following optimization problem for $k$ classes:
\begin{eqnarray}\label{eq:multiclasscramer}
\nonumber  \min_{\vec{w}_m\in\R^d, \xi_i\in\R} &\ \ \ \ & \frac{1}{2}\sum_{m=1}^k\vec{w}_m^T\vec{w}_m + C\sum_{i=1}^l\xi_i \\
  \mbox{subject to } & & \left\langle\vec{w}_{y_i},\Phi(\vec{x}_i)\right\rangle - \left\langle\vec{w}_t,\Phi(\vec{x}_i)\right\rangle \ge 1 -\delta_{y_i,t} - \xi_i \\
\nonumber  & & \xi_i \ge 0,\ i=1\dots l,\ \ t = 1\dots k
\end{eqnarray}
where $\delta_{i,j}$ is the Kronecker delta function. The corresponding decision function is:
$$argmax_mf_m(\vec{x}) = argmax_m\left\langle\vec{w}_m,\Phi(\vec{x})\right\rangle$$
The original formulation addresses the classification without taking into account the bias terms $b_i$ (for each of the $l$ classes). These can be easily included in the formulation using additional constraints (see, for instance, \cite{hsu:02:comparison}). Crammer's formulation is among the most compact optimization problem formulations for multi-class classification problem.

A common issue of the single optimization problem formulations for multi-class classification is the large number of variables involved. For instance, \eqref{eq:multiclasscramer}, although a compact formulation, includes $l\times k$ variables (not taking into account $b_i$'s, if included), which yields large computation complexity. In \cite{he:12:simplified}, Crammer's formulation is extended by relaxing its constraints and subsequently solving a single $l$-variable quadratic programming problem for multi-class classification.

\subsection{One-versus-rest approach}
The one-versus-rest approach \cite{vapnik:98:statistical, bottou:94:comparison, szedmak:04:multiclass} is an indirect method relying on binary classifiers as follows. For each class $t \in \{1, \dots, k\}$ a binary classifier $f_t$ is created between class $t$ (as positive examples in the training set) and all the other classes, $\{1, \dots, t-1, t+1, \dots, k\}$ (all as negative examples in the training set). The corresponding decision function is then:
$$f(\vec{x}) = argmax_{1\le t\le k}f_t(\vec{x})$$
That is, the class label is determined by the binary classifier that gives maximum output value (the winner among all classifiers). A well-known shortcoming of the one-versus-rest approach is the highly imbalanced training set for each binary classifier (the more classes, the bigger the imbalance). Assuming equal number of training examples for all classes, the ratio of positive to negative examples for each binary classifier is $1 / (k-1)$. The symmetry of the original problem is lost and the classification results may be dramatically affected (especially for sparse classes).

\subsection{One-versus-one approach}
The one-versus-one approach (\cite{knerr:90:single, friedman:96:another, krebel:99:pairwise, park:07:efficient} or the improvement by Platt et al. \cite{platt:99:dags}) aims to remove the imbalance problem of the one-versus-rest method by training binary classifiers strictly with data in the two classifier's classes. For each pair of classes, $s,t\in\{1, \dots, k\}$ a binary classifier $f_{s,t}$ is created. This classifier is trained using all data in class $s$ as positive examples and all data in class $t$ as negative examples, hence all balanced binary classifiers. Each binary classifier is the result of a smaller optimization problem, at the cost of producing $k(k-1)/2$ classifiers.  The corresponding decision function is based on majority voting. All classifiers $f_{s,t}$ are used on an input data item $\vec{x}$ and each class appears in exactly $k-1$ classifiers, hence an opportunity for up to $k-1$ votes out of the $k(k-1)/2$ binary classification rounds. The class with the majority of votes is the winner.

An improvement on the number of voting rounds was proposed by Platt et al. in \cite{platt:99:dags}. Their method, called Directed Acyclic Graph SVM (DAGSVM),  forms a decision-graph structure for the testing phase and it takes exactly $k-1$ individual voting rounds to decide the label of a data item $\vec{x}$. In a nutshell, DAGSVM uses one binary classifier at the time and subsequently removes the losing class from all subsequent classifications. There is no particular criterion on the order of using each binary classifier in this process.

\section{Divide and Conquer SVM (DCSVM)}\label{sec:dcsvm}
As noted in the introduction and illustrated in \Cref{fig:artificial1}, the key idea is that any binary classifier may, in practice, separate more than two classes. Which raises a natural question: which classes are separated (and with what accuracy) by each binary classifier? DCSVM combines the one-versus-one method's simplicity of producing balanced, fast binary classifiers with the classification speed of the DAGSVM's decision graph. The essential difference consists of producing the most efficient decision tree capable of delivering the decision in at most $k-1$ steps in the worst case scenario, or $O(\log k)$ steps in the best case scenario.

\subsection{DCSVM training}
Let us introduce some notations and then we will proceed to the formal description of the algorithm. Given a data set $D$ of $k$ classes (labels) where to each data item $\vec{x}\in D$ has been assigned a label $l\in\{1,\dots, k\}$, we want to construct a decision function $dcsvm : D \rightarrow \{1,\dots, k\}$ so that $dcsvm(\vec{x}) = l$, where $l$ is the corresponding label of $\vec{x}\in D$. As usual, by considering a split $D = R\cup T$ of the data set $D$ into two disjoint sets $R$ (the training set) and $T$ (the test set), we will be using the data in $R$ to construct our decision function $dcsvm()$ and then the data in $T$ to measure its accuracy. Furthermore, we consider $R = R_1\cup R_2\cup\dots\cup R_k$ as an union of disjoint sets $R_l$, where each $x\in R_l$ has label $l$, $l = 1,\dots, k$. (Similarly, we consider $T = T_1\cup T_2\cup\dots\cup T_k$ as a union of disjoint sets $T_l$, where each $x\in T_l$ has label $l$, $l = 1,\dots, k$.)

Let $svm_{i,j} : D\rightarrow\{i,j\}$, be a SVM binary classifier created using the training set $R_i\cup R_j$, $i < j$ and $i = 1,\dots,k-1, j = 2,\dots,k$. There are $k(k-1)/2$ such one-versus-one binary classifiers. We must clearly specify that the $svm()$ decision function we consider here is not the ideal one, but the practical one, likely affected by misclassification errors. That is, for some $\vec{x}\in R_i\cup T_i$, we may have that $svm_{i,j}(\vec{x}) = j$.

Our goal is to create the $dcsvm()$ decision function that uses a minimal number of binary decisions for $k$-classes classification, while not sacrificing the classification accuracy. We define next a few measures we use in the process of identifying the shortest path to a multi-class classification decision.

\begin{definition}[Class Predictions Likelihoods]
The class predictions likelihoods of a SVM binary classifier $svm_{i,j}(\cdot)$ for a label $l\in\{1,\dots,k\}$, denoted respectively as  $\mathcal{C}_{i,j}(l, i)$ and $\mathcal{C}_{i,j}(l, j)$,  are:
\begin{eqnarray*}
  \mathcal{C}_{i,j}(l, i) &=& \frac{|\{\vec{x}\in R_l\;|\;svm_{i,j}(\vec{x}) = i\}|}{|R_l|} \\
  \mathcal{C}_{i,j}(l, j) &=& \frac{|\{\vec{x}\in R_l\;|\;svm_{i,j}(\vec{x}) = j\}|}{|R_l|} = 1 - C_{i,j}(l,i)
\end{eqnarray*}
\end{definition}
Each class prediction likelihood represents the expected outcome likelihood for $i$ or $j$ when a binary classifier $svm_{i,j}(\cdot)$ is used for prediction on all data items in $R_l$. These likelihoods are computed for each binary classifier and each class in the training data set.

%We then arrange all classes predictions likelihoods for a given binary classifier $svm_{i,j}(\cdot)$ in a sequence associated to the respective classifier.
%\begin{definition}[SVM Predictions Likelihoods]
%The SVM predictions likelihoods of a SVM classifier $svm_{i,j}(\cdot)$, denoted as $\mathcal{P}_{i,j}$, is the sequence of all class predictions likelihoods pairs for this classifier:
%$$\mathcal{P}_{i,j} = \left\{\left(\mathcal{C}_{i,j}(1, i), \mathcal{C}_{i,j}(1, j)\right), \left(\mathcal{C}_{i,j}(2, i), \mathcal{C}_{i,j}(2, j)\right), \dots, \left(\mathcal{C}_{i,j}(k, i), \mathcal{C}_{i,j}(k, j)\right)\right\}$$
%\end{definition}
All pairs of likelihood predictions, for every binary classifier $svm_{i,j}(\cdot)$ and classes are stored in a table, as follows.
\begin{definition}[All-Predictions Table] We arrange all classes predictions likelihoods in rows (corresponding to each binary classifier $svm_{i,j}$) and columns (corresponding to each class $1,\dots,k$) to form a table $\mathcal{T}$ where each entry is given by a pair of predictions likelihoods as follows:
$$\mathcal{T}[svm_{i,j},l]=\left(\mathcal{C}_{i,j}(l, i), \mathcal{C}_{i,j}(l, j)\right)$$
\end{definition}
\Cref{fig:allpredict} shows the \emph{All-Predictions Table} computed for the \emph{glass} data set in \cite{Dua:2017}. The data set contains 6 classes, labeled as 1, 2, 3, 5, 6, and 7. Each row corresponds to a binary classifier $svm_{1,2},\cdots,svm_{6,7}$ and the columns correspond to the class labels. Each table cell contains a pair of likelihood predictions (as percentages) for the row classifier and class column. For instance, $\mathcal{C}_{1,2}(1, 1) = 100\%$, $\mathcal{C}_{1,2}(1, 2) = 0\%$ and $\mathcal{C}_{1,6}(2, 1) = 91.8\%$, $\mathcal{C}_{1,6}(2, 6) = 8.2\%$.
\begin{figure}[ht]
\centering
\includegraphics[scale=0.55]{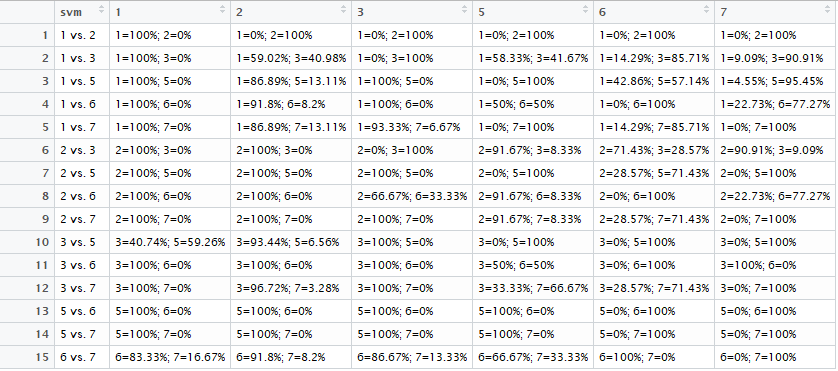}
\caption{The \emph{All-Predictions} table for the \emph{glass} data set in \cite{Dua:2017}}\label{fig:allpredict}
\end{figure}

We define next two measures for the quality of the classification of each $svm_{i,j}(\cdot)$. The purity index measures how good the binary classifier is for classifying all classes as $i$ or $j$ for a given precision threshold $\theta$. In a nutshell, a class $l$ is classified as ``definitely" $i$ by $svm_{i,j}(\cdot)$ if $\mathcal{C}_{i,j}(l, i) \ge 1-\theta$; as ``definitely" $j$ if $\mathcal{C}_{i,j}(l, j) \ge 1-\theta$; otherwise, it is classified as ``undecided" $i$ or $j$. The purity index counts how many ``undecided" decisions a binary classifier produces. The lower the index, the better the separation. The balance index measures how ``balanced" a separation is in terms of the number of classes predicted as $i$ and $j$. The larger the index, the better.
\begin{definition}[SVM Purity and Balance Indexes]\label{def:purityandbalance}
For an accuracy threshold $\theta$ of a SVM classifier $svm_{i,j}(\cdot)$, we define: \\
- the purity index, denoted as $\mathcal{P}_{i,j}(\theta)$, as:
$$\mathcal{P}_{i,j}(\theta) = \left(\sum_{l = 1}^{k}\left(\chi_\theta(\mathcal{C}_{i,j}(l, i)) + \chi_\theta(\mathcal{C}_{i,j}(l, j))\right)\right)-k$$
- the balance index, denoted as $\mathcal{B}_{i,j}(\theta)$, as:
$$\mathcal{B}_{i,j}(\theta) = min\left(k-\sum_{l = 1}^{k}\chi_\theta(\mathcal{C}_{i,j}(l, j)), k-\sum_{l = 1}^{k}\chi_\theta(\mathcal{C}_{i,j}(l, i))\right)$$
where $\chi_\theta$ is the step function:
$$\chi_\theta(x) = \left\{\begin{array}{ll}1\ \ \ \ &\mbox{if }x > \theta\\0 & \mbox{if }x \le \theta\end{array}\right.$$
\end{definition}
For instance, the purity index for row $svm_{1,6}$ and threshold $\theta = 0.05$ in \Cref{fig:allpredict} is:
$$\mathcal{P}_{1,6}(0.05) = ((1+0) + (1+1) + (1+0) + (1+1) + (0+1) + (1+1)) - 6 = 3$$
and indicates that 3 of the classes (namely 2, 5, and 7) are undecided when the required precision is at least $\theta = 5\%$.

For accuracy threshold $\theta = 0.05$, the balance index for row $svm_{1,2}$ in \Cref{fig:allpredict} is $\mathcal{B}_{1,2} = 1$ and for row $svm_{5,6}$ is $\mathcal{B}_{5,6} = 2$.

The SVM score, defined next, is a measure of the precision of the binary classifier $svm_{i,j}(\cdot)$ for classifying classes $i$ and $j$. The higher the score the better the classifier precision.
\begin{definition}[SVM Score]\label{def:score}
The score of a SVM classifier $svm_{i,j}(\cdot)$, denoted as $\mathcal{S}_{i,j}$, is
$$\mathcal{S}_{i,j} = \frac{\mathcal{C}_{i,j}(i, i) + \mathcal{C}_{i,j}(j, j)}{2}$$
\end{definition}

For instance, the table in \Cref{fig:allpredict} shows that
$$\mathcal{S}_{1,2} = \frac{\mathcal{C}_{1,2}(1, 1) + \mathcal{C}_{1,2}(2, 2)}{2} = \frac{100\%+100\%}{2} = 100\%.$$

\begin{algorithm}
    \caption{DCSVM training}
    \label{alg:dcsvm}
    \begin{algorithmic}[1] % The number tells where the line numbering should start
        \Procedure{TrainDCSVM}{$R_1,\dots,R_k,\theta$} \Comment{Creates DCSVM classifier}\\
            \textbf{Input}: $R = R_1,\dots,R_k$: data set, $\theta$: accuracy threshold\\
            \textbf{Output}: $dcsvm()$
            \State {$svm_{i,j}\gets$ train SVM with $R_i \cup R_j$, $i=1,\dots,k-1$, $j=2,\dots,k$, $i < j$}
            \State{$\mathcal{T}[svm_{i,j},l]=\left(\mathcal{C}_{i,j}(l, i), \mathcal{C}_{i,j}(l, j)\right)$, for all $svm_{i,j}$, $l = 1,\dots,k$}
            \State{//Recursively construct a binary decision tree}
            \State{//with each node associated with a $svm_{i,j}$ binary classifier}
            \State{$dcsvm\gets$ empty binary tree}
            \State{$dcsvm.root \gets $ new tree node}
            \State{\textsc{DCSVM-subtree}($dcsvm.root, \mathcal{T},\theta$)}
            \State \textbf{return} $dcsvm$\Comment{Returns the decision tree}
        \EndProcedure
    \end{algorithmic}

    \begin{algorithmic}[1] % The number tells where the line numbering should start
        \Procedure{DCSVM-subtree}{$pnode, \mathcal{T},\theta$} \Comment{Creates subtree routed at $pnode$}\\
            \textbf{Input}: $pnode$: current parent node, $\mathcal{T}$: current predictions table,  $\theta$: accuracy threshold\\
            \textbf{Output}: recursively constructs sub-tree rooted at $pnode$
            \State{$svm_{i,j} \gets $ optimal $svm$ in $\mathcal{T}$, for given $\theta$}
            \State{$pnode[svm] \gets svm_{i,j}$}
            \State{$listi\gets$ classes labeled as $i$ or undecided by $svm_{i,j}$}
            \State{$listj\gets$ classes labeled as $j$ or undecided by $svm_{i,j}$}
            \If{$length(listi) = 1$}\Comment{reached a leaf}
               \State{$pnode.leftnode\gets $ tree-node(label in $listi$)}
            \Else
               \State{$\mathcal{T}.left\gets\mathcal{T}$ minus $svm_{m,n}, m\in listj$ or $n\in listj$ rows, and columns of classes not in $listi$}
               \State{$pnode.left\gets$ new tree node}
               \State{\textsc{DCSVM-subtree}($pnode.left, \mathcal{T}.left,\theta$)}
            \EndIf
            \If{$length(listj) = 1$}\Comment{reached a leaf}
               \State{$pnode.rightnode\gets $ tree-node(label in $listj$)}
            \Else
               \State{$\mathcal{T}.right\gets\mathcal{T}$ minus $svm_{m,n}, m\in listi$ or $n\in listi$ rows, and columns of classes not in $listj$}
               \State{$pnode.right\gets$ new tree node}
               \State{\textsc{DCSVM-subtree}($pnode.right, \mathcal{T}.right,\theta$)}
            \EndIf
        \EndProcedure
    \end{algorithmic}
\end{algorithm}

Algorithm~\ref{alg:dcsvm} describes the DCSVM training and proceeds as follows. In the main procedure, \textsc{TrainDCSVM}, the SVM binary classifiers for all class pairs are trained (line 4) and the predictions likelihoods are stored in the predictions table (line 5). The decision function $dcsvm$ is created as an empty tree (line 8) and then recursively populated in the \textsc{DCSVM-subtree} procedure (line 10). The recursion procedure creates a left and/or a right node at each step (lines 12 and 19, respectively) or may stop with creating a left and/or a right label (lines 9 and 16, respectively). Each new node is associated to the binary $svm_{i,j}$ that is the decider at that node (line 5), or with a class label if an end node (lines 9 and 16).

An important part of the \textsc{DCSVM-subtree} procedure is choosing the ``optimal" $svm$ from a current predictions likelihoods table (line 4). For this purpose, we use the SVM Purity Index, Balance, and Score from Definitions \ref{def:purityandbalance} and \ref{def:score}, respectively. The order these measures are used may influence the decision tree shape and precision. If Score is used then the Purity and Balance Indexes are used to break a tie, the resulting tree favors accuracy over the speed of decisions (may yield bushier trees). If Purity and Balance Indexes are used first, then Score, if a tie, the resulting tree may be more balanced. The decision speed is favored while possibly sacrificing accuracy.

A $dcsvm$ decision tree for the \emph{glass} data set is shown in \Cref{fig:tree}. Clearly, the algorithm may produce highly unbalanced $dcsvm$ decision trees (when some classes are decided faster than others) or very balanced decision trees (when most of class labels are leaves situated at about same depth). Regardless of outcome, the following result is almost immediate.
\begin{proposition}\label{prop:worstcase}
The $dcsvm$ decision tree constructed in Algorithm~\ref{alg:dcsvm} has depth at most $k-1$.
\end{proposition}
\begin{proof}
The lists of classes labeled $i$ and $j$ (lines 6, 7 in \textsc{DCSVM-subtree} procedure) contain at least one label each: $i$ or $j$, respectively. Once a class column is removed from $\mathcal{T}$ at some tree node $n$, it will not appear again in a node or leaf in the subtree rooted at that node $n$. Hence with each recursion the number of classes decreases by at least one (lines 11, 18) from $k$ to 2, ending the recursion with a left or a right label node in lines 9 or 16, respectively.
\end{proof}
Notice that a scenario where each $dcsvm$ decision tree label has depth $k$ is possible in practice: when no $svm_{i,j}$ binary classifier is a good separator for classes other than $i$ and $j$ (and therefore at each node only classes $i$ and $j$ are separated, while the other are undecided and will appear in both left and right branches). We call this the worst case scenario, for obvious reasons. The opposite case scenario is also possible in practice: each $svm_{i,j}$ separates all classes into two disjoint lists of about same lengths. The $dcsvm$ decision tree is also very balanced in this case, but a lot smaller.
\begin{proposition}
The $dcsvm$ decision tree constructed in Algorithm~\ref{alg:dcsvm} when each $svm_{i,j}$ produced balanced, disjoint separation between all classes has depth at most $\lceil\log k\rceil$.
\end{proposition}\label{prop:bestcase}
\begin{proof}
Clearly, this is a case scenario where at each recursion step a node is created such that half of the classes are assigned to the left subtree and the other half to the right subtree. This produces a balanced binary tree with $k$ leaves, hence of depth at most $\lceil\log k\rceil$.
\end{proof}

\begin{algorithm}
    \caption{DCSVM classifier}
    \label{alg:dcsvmdecision}
    \begin{algorithmic}[1] % The number tells where the line numbering should start
        \Procedure{DCSVMclassify}{$dcsvm,\vec{x}$} \Comment{Produces DCSVM classification}\\
            \textbf{Input}: $dcsvm$: decision tree; $\vec{x}$: data item\\
            \textbf{Output}: Label of data item $\vec{x}$
            \State {$node\gets dcsvm.root$}
            \While{$node$ not a leaf}\Comment{Visits the decision tree nodes towards a leaf}
              \State{$svm_{i,j}\gets node[svm]$}\Comment{Retrieves the $svm$ associated to current node}
              \State{$label\gets svm_{i,j}(\vec{x})$}
              \If{$label = i$}\State{$node\gets node.left$}
              \Else\State{$node\gets node.right$}\EndIf
            \EndWhile
            \State \textbf{return} label of $node$\Comment{Returns the leaf label}
        \EndProcedure
    \end{algorithmic}
\end{algorithm}
The DCSVM classifier Algorithm~\ref{alg:dcsvmdecision} relies on the $dcsvm$ decision tree produced by Algorithm~\ref{alg:dcsvm} to take any data item $\vec{x}$ and predict its label. The algorithm starts at the decision tree root node (line 4) then each node's associated $svm$ predicts the path to follow (lines 6--12) until a leaf node is reached. The label of the leaf node is the DCSVM's prediction (line 14) for the input data item $\vec{x}$. An example of a prediction path in a $dcsvm$ tree is illustrated in \Cref{fig:tree} (b).

Propositions \ref{prop:worstcase} and \ref{prop:bestcase} directly justify the following result.
\begin{theorem}
The Algorithm~\ref{alg:dcsvmdecision} performs multi-class classification of any data item $\vec{x}$ in at most $k-1$ binary decisions steps (in the worst case scenario) and at most $\lceil\log k\rceil$ binary decision steps (in the best case scenario).
\end{theorem}
We illustrate next how the $dcsvm$ decision tree is created and how a prediction is computed using a working example.

\subsection{A working example}
We use the \emph{glass} data set \cite{Dua:2017} to illustrate DCSVM at work. This data set contains 6 classes, labeled 1, 2, 3, 5, 6, and 7 (notice there is no label 4). Consequently, $6 * (6 - 1) / 2 = 15$ binary $svm$ classifiers are created and then the ``all predictions likelihoods" table $\mathcal{T}$ is computed (\Cref{fig:allpredict}). Let us choose the accuracy threshold $\theta = 0$, for simplicity. That is, a class $l$ is classified by an $svm_{i,j}$ as only $i$ if $svm_{i,j}$ predicts that all data items in $R_l$ have class $i$; $l$ is classified as only $j$ if $svm_{i,j}$ predicts that all data items in $R_l$ have class $j$; else, $l$ is undecided and will appear on both sides of the decision tree node associated with $svm_{i,j}$.

\begin{figure}[ht]
\centering
\includegraphics[scale=0.65]{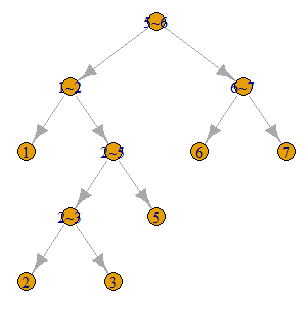}
\includegraphics[scale=0.65]{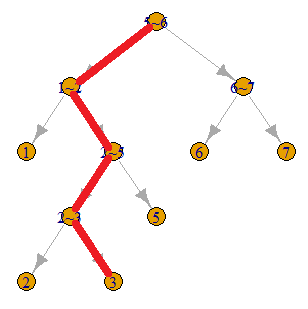}\\
(a)\hspace{5cm}(b)
\caption{DCSVM decision tree (a) with a decision process example (b), for the \emph{glass} data set}\label{fig:tree}
\end{figure}

We follow next the \textsc{DCSVM-subtree} procedure in Algorithm~\ref{alg:dcsvm} and construct the $dcsvm$ decision tree. Notice that $S_{i,j} = 100\%$ for all $svm_{i,j}$, so score does not matter for choosing the optimal $svm_{i,j}$ in line 4. The choice will be solely based on the purity and balance indexes. Table~\ref{table:measures} shows all values for these measures for the initial predictions likelihoods table.
\begin{table}
  \caption{$svm$ optimality measures for \emph{glass} data set and the initial \emph{All-Predictions} table}\label{table:measures}
  \centering
\begin{tabular}{|r|c|c|c|c|}
  \hline
   & $svm_{i,j}$ & $\mathcal{P}_{i,j}(0)$ & $\mathcal{B}_{i,j}(0)$ & $\mathcal{S}_{i,j}$ \\\hline
  1. & $svm_{1,2}$ & 0 & 1 & 100\% \\
  2. & $svm_{1,3}$ & 4 & 1 & 100\% \\
  3. & $svm_{1,5}$ & 3 & 1 & 100\% \\
  4. & $svm_{1,6}$ & 3 & 1 & 100\% \\
  5. & $svm_{1,7}$ & 3 & 1 & 100\% \\
  6. & $svm_{2,3}$ & 3 & 1 & 100\% \\
  7. & $svm_{2,5}$ & 1 & 2 & 100\% \\
  8. & $svm_{2,6}$ & 3 & 1 & 100\% \\
  9. & $svm_{2,7}$ & 2 & 1 & 100\% \\
  10. & $svm_{3,5}$ & 2 & 1 & 100\% \\
  11. & $svm_{3,6}$ & 1 & 1 & 100\% \\
  12. & $svm_{3,7}$ & 3 & 1 & 100\% \\
  13. & $svm_{5,6}$ & 0 & 2 & 100\% \\
  14. & $svm_{5,7}$ & 0 & 2 & 100\% \\
  15. & $svm_{6,7}$ & 4 & 1 & 100\% \\
  \hline
\end{tabular}
\end{table}
The table shows rows 1, 13, and 14 as candidates with minimum purity indexes. Then a tie between rows 13 and 14 as the winners among these. Row 13 comes first and hence $svm_{5,6}$ is selected as the root node. \Cref{fig:tree}~(a) shows the full decision tree, with $svm_{5,6}$ as the root node. Subsequently, $svm_{5,6}$ labels classes 1, 2, 3, and 5 as ``5" (left), and classes 6 and 7 as ``6" (right). The algorithm continues recursively with classes $\{1, 2, 3, 5\}$ to the left, and classes $\{6, 7\}$ to the right. The right branch will be completed immediately with one more tree node (for $svm_{6,7}$) and two corresponding leaf nodes (for labels 6 and 7).

For the left branch the algorithm will proceed with a reduced \emph{All-Predictions} table: rows 4, 5, 8, 9, 11, 12, 13, 14, and 15 and columns for classes 6 and 7 are removed. The optimality measures will be subsequently computed for all $svm$ and classes still in competition (1, 2, 3, and 5) in the left branch. The corresponding measures are given in Table~\ref{table:measures2} (for an easier identification, the indices in the first column are kept the same as the original indices in the \emph{All-Predict} table in \Cref{fig:allpredict}).
\begin{table}
  \caption{Optimality measures in the second step of creating the decision tree in \Cref{fig:tree} (b)}\label{table:measures2}
  \centering
\begin{tabular}{|r|c|c|c|c|}
  \hline
   & $svm_{i,j}$ & $\mathcal{P}_{i,j}(0)$ & $\mathcal{B}_{i,j}(0)$ & $\mathcal{S}_{i,j}$ \\\hline
  1. & $svm_{1,2}$ & 0 & 1 & 100\% \\
  2. & $svm_{1,3}$ & 2 & 1 & 100\% \\
  3. & $svm_{1,5}$ & 1 & 1 & 100\% \\
  6. & $svm_{2,3}$ & 1 & 1 & 100\% \\
  7. & $svm_{2,5}$ & 0 & 1 & 100\% \\
  10. & $svm_{3,5}$ & 2 & 1 & 100\% \\
  \hline
\end{tabular}
\end{table}
There is a tie between $svm_{1,2}$ and $svm_{2,5}$, and $svm_{1,2}$ is being used first. A node is consequently created, with a leaf as a left child. The rest of the tree is subsequently created in the same manner.

\section{Experimental results}\label{sec:experimental}
\begin{table}[htbp]
\caption{Data sets}\label{tab:datasets}
\begin{center}
\begin{tabular}{|r|l|c||r|l|c|}
  \hline
  No & Dataset & Classes & No & Dataset & Classes \\\hline
  1. & artificial & 6 & 8. & covertype & 7 \\
  2. & iris & 3 & 9. & svmguide4 & 6 \\
  3. & segmentation & 7 & 10. & vowel & 11 \\
  4. & heart & 5 & 11. & usps & 10 \\
  5. & wine & 3 & 12. & letter & 26 \\
  6. & wine-quality & 6 & 13. & poker & 10 \\
  7. & glass & 6 & 14. & sensorless & 11 \\
  \hline
\end{tabular}
\end{center}
\end{table}
\begin{figure}[ht]
\centering
\includegraphics[scale=0.5]{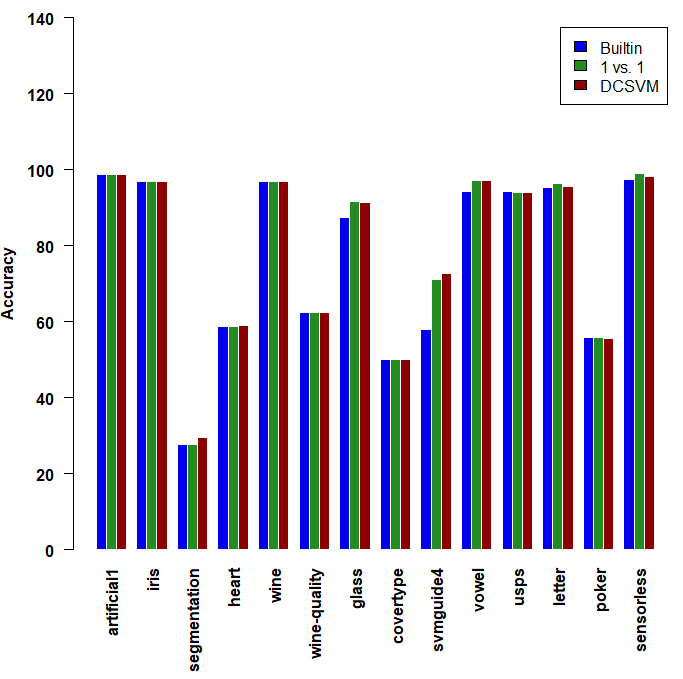}
\caption{Multi-class prediction accuracy comparison: built-in SVM, one-versus-one, and DCSVM}\label{fig:barplot1}
\end{figure}

We implemented DCSVM in R v3.4.3 using the e1071 library \cite{dimitriadou:18:e1071}, running on Windows 10, 64-bit Intel Core i7 CPU @3.40GHz, 16GB RAM. For testing, we used 14 data sets from the UCI repository \cite{Dua:2017} (as listed in Table~\ref{tab:datasets}).
We performed three sets of experiments: (i) multi-class prediction accuracy comparison, (ii) prediction performance in terms of speed (time and number of binary decisions) and resources (number of support vectors), and (iii) DCSVM performance comparisons for different data sets and accuracy threshold parameter values. For the first set of experiments we compared three multi-class predictors: the built-in multi-class SVM (from the \emph{e1071} library), our R implementation of one-versus-one, and the R implementation of DCSVM. For a fair comparison, in the second set of experiments we compared only the R implementations of one-versus-one and DCSVM. The built-in multi-class SVM would benefit of the inherent speed of native code it relies on. Finally, the third set of experiments focused on the DCSVM's R implementation performance and fine tuning.
\subsection{Accuracies comparison: built-in multi-class SVM, one-versus-one, and DCSVM}
The main goal of DCSVM is to improve multi-class prediction performance while not sacrificing the prediction accuracy. The first experimental results compare multi-class prediction accuracy of: (i) built-in SVM multi-class prediction (in the \emph{e1071} package), (ii) one-versus-one implementation in R, and (iii) DCSVM implementation in R. For the experiment, we used cross-validation with 80\% data for training and 20\% for testing, for each data set. We ran 10 trials and averaged the results. The results are displayed in \Cref{fig:barplot1} and show no significant differences between the three methods.
\subsection{Prediction performance comparison}
\begin{figure}[ht]
\centering
\includegraphics[scale=0.3]{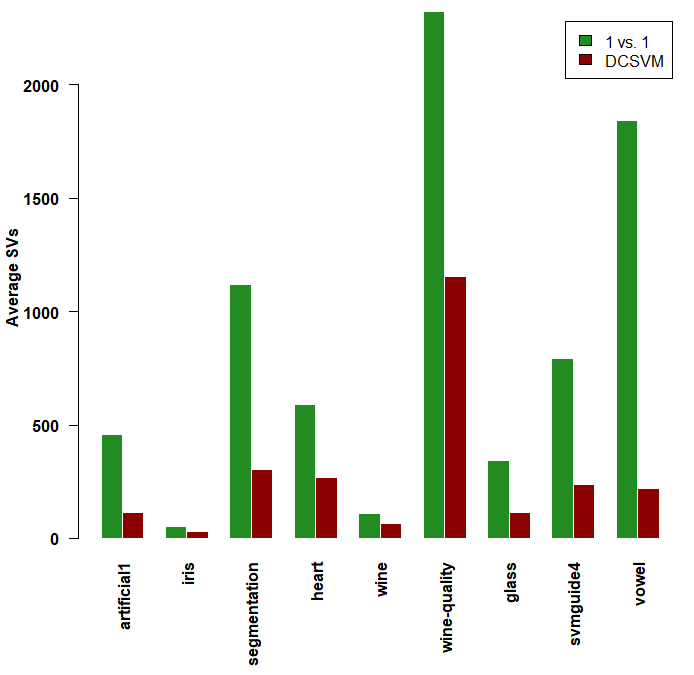}
%\caption{Average Support Vectors}\label{fig:barplot2-1}
%\end{figure}
%
%\begin{figure}[ht]
%\centering
\includegraphics[scale=0.3]{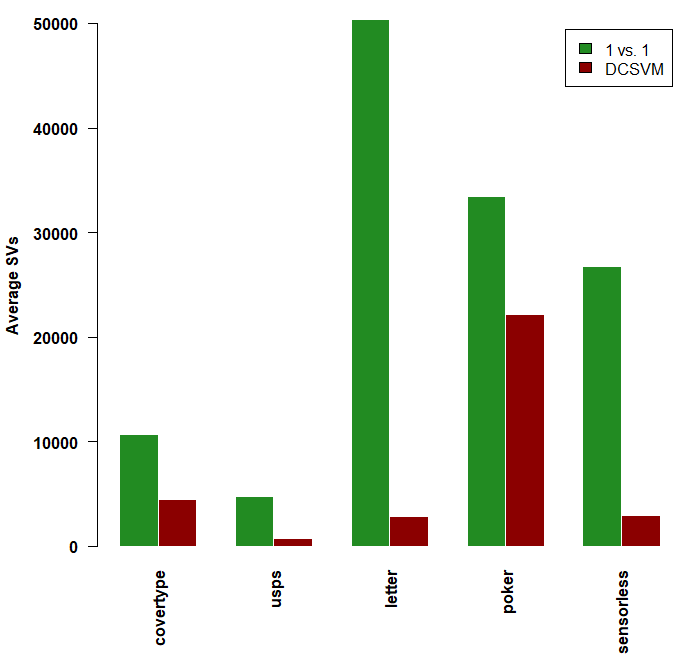}
\caption{Average number of Support Vectors for multi-class predictions}\label{fig:barplot2-2}
\end{figure}
\begin{figure}[ht]
\centering
\includegraphics[scale=0.3]{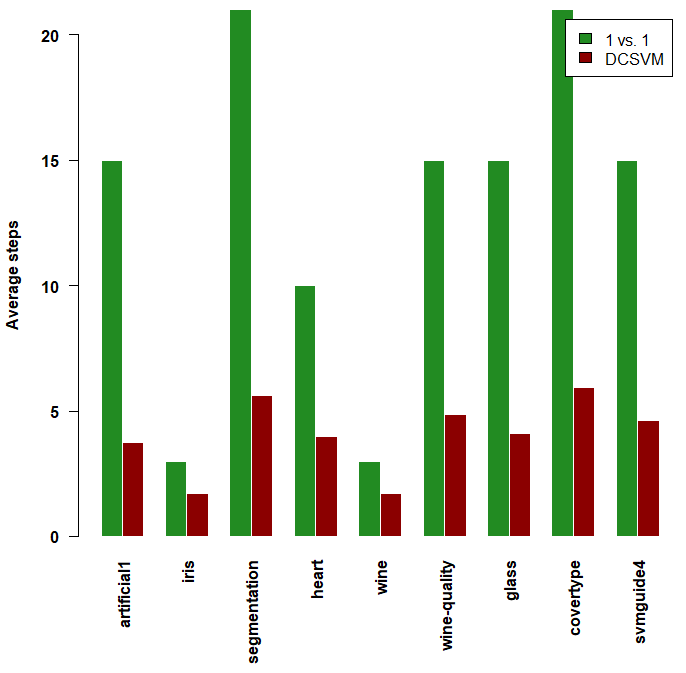}
%\caption{Average number of decisions}\label{fig:barplot3-1}
%\end{figure}
%
%\begin{figure}[ht]
%\centering
\includegraphics[scale=0.3]{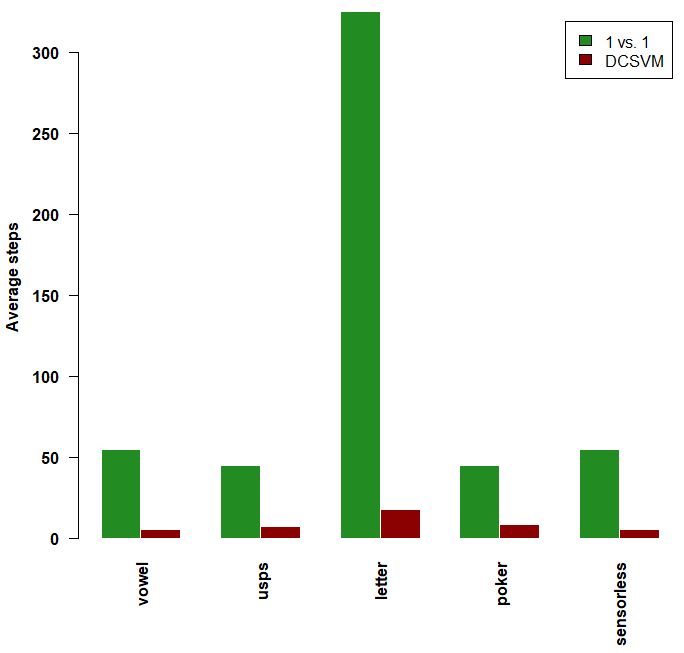}
\caption{Average number of binary decisions for multi-class predictions}\label{fig:barplot3-2}
\end{figure}
\begin{figure}[ht]
\centering
\includegraphics[scale=0.3]{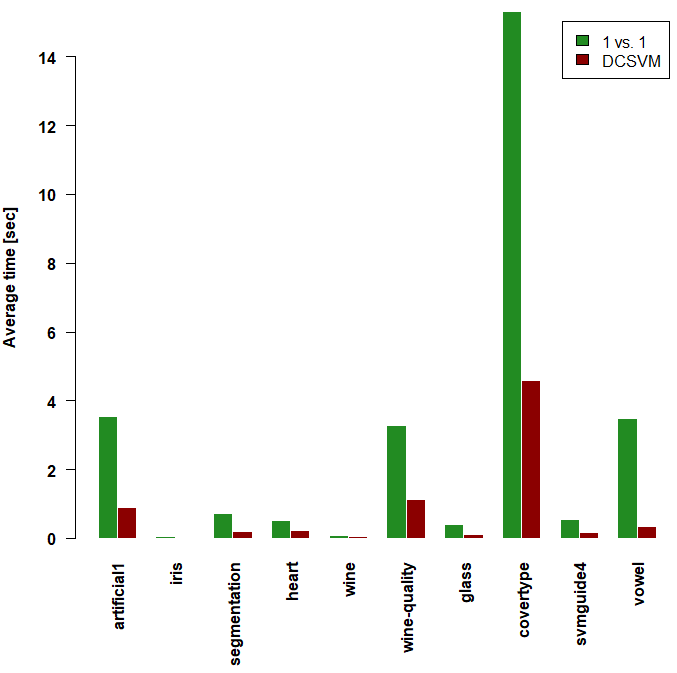}
%\caption{Average prediction times}\label{fig:barplot4-1}
%\end{figure}
%
%\begin{figure}[ht]
%\centering
\includegraphics[scale=0.3]{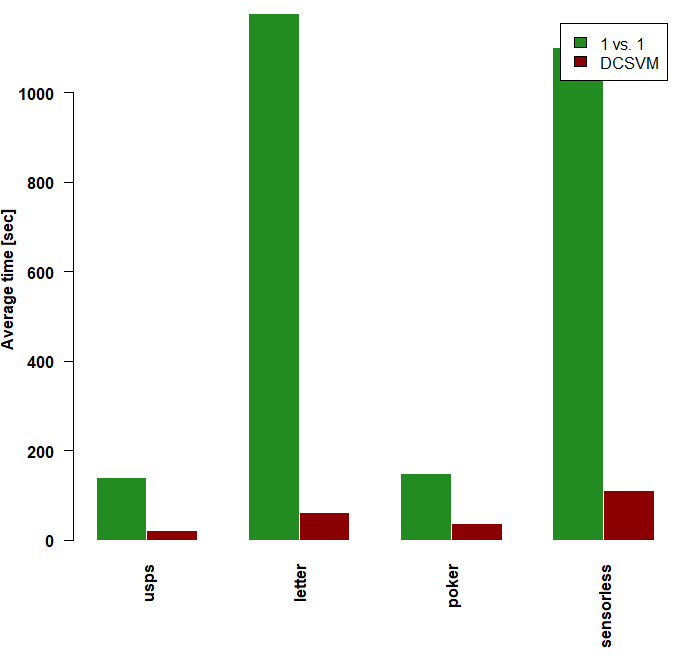}
\caption{Average prediction times for multi-class predictions}\label{fig:barplot4-2}
\end{figure}
For this purpose, we compared the R implementations of one-versus-one method and DCSVM.
We analyzed prediction performance in three aspects: the average number of support vectors, the average number of binary decisions, and time. The number of support vectors used was computed by summing up all support vectors from every binary decider, over all steps of binary decisions until the multi-class prediction was achieved. The number of such support vectors is clearly proportional not only to the number of decision steps (which are illustrated separately), but also to the configuration of data separated by each binary classifier. The corresponding performance results are presented in \Cref{fig:barplot2-2}, \Cref{fig:barplot3-2}, and \Cref{fig:barplot4-2}, respectively. Due to large variations in size between the data sets we used, we split the data sets into two size-balanced groups and displayed each graph side-by-side for each group. DCSVM significantly outperforms one-versus-one, clearly being much less computationally intensive (number of support vectors for prediction) and faster (number of binary decisions and prediction times).

From the first two sets of experimental results we can conclude already that DCSVM achieved the goal of being a faster multi-class predictor without sacrificing prediction accuracy.

\subsection{DCSVM performance fine tuning}
In this set of experiments we analyze in close detail DCSVM's performance in terms of the accuracy threshold parameter.
%The graph of accuracy versus threshold is shown in \Cref{fig:accuracyvsthshold}.
\begin{figure}[ht]
\centering
\includegraphics[scale=0.35]{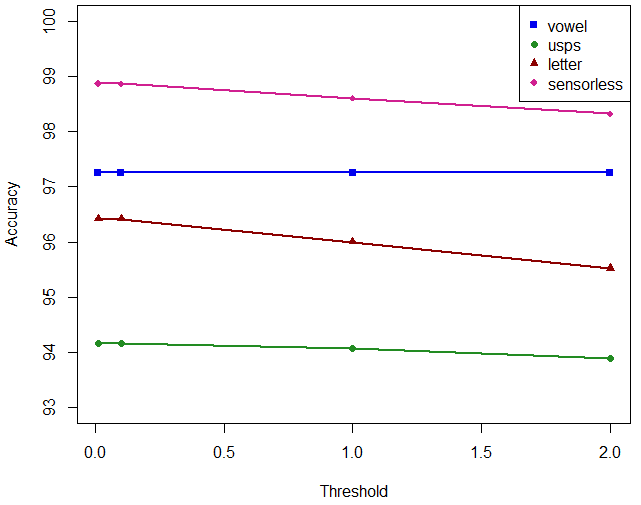}
\includegraphics[scale=0.35]{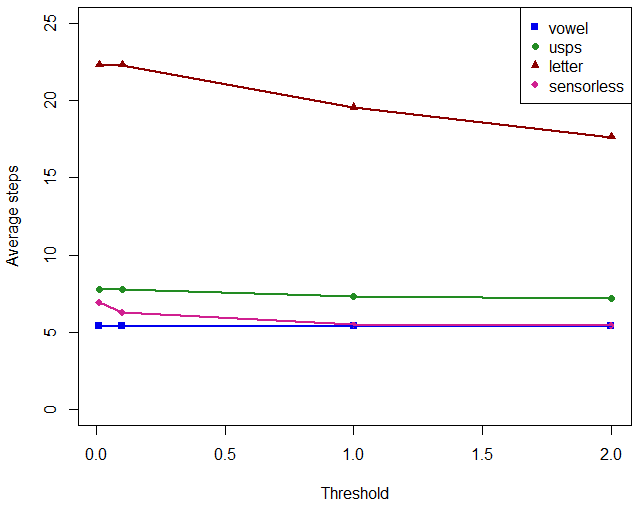}
\caption{Accuracy and the average number of prediction steps for different thresholds}\label{fig:accuracyvsthshold}
\end{figure}
\Cref{fig:accuracyvsthshold} shows the trade-off between accuracy (left) and the average prediction steps (right) with various threshold values. Clearly, the accuracy threshold parameter permits a trade-off between accuracy and speed. However, this is largely data dependent. The more separable the data is, the less influence the threshold has on speed. For less separable data (such as the \emph{letter} data set), fine adjustment of the threshold permits trade-off between prediction accuracy and prediction speed. This is not the case for the \emph{vowel} data set, which is highly separable: changes in the threshold influence neither the accuracy of prediction nor the average number of prediction steps.
\begin{figure}[ht]
\centering
\includegraphics[scale=0.35]{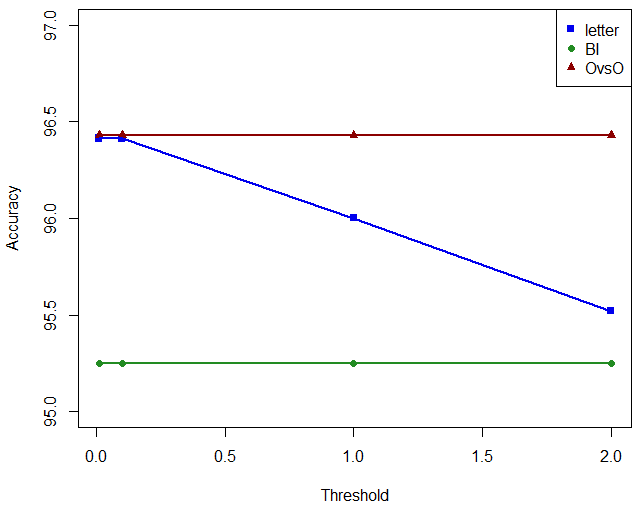}
\caption{Accuracy for predicting ``letter" with each method, for different split thresholds}\label{fig:accbymethvsthshold}
\end{figure}

\Cref{fig:accbymethvsthshold} shows how DCSVM accuracy compares to other multi-class classification methods (BI = built-in, OvsO = one-against-one) for various threshold values. For less separable data (such as \emph{letter}) DCSVM's accuracy drops sharply with the threshold (starting at some small threshold value) compared to the accuracy of one-against-one method, which we found to perform better than the built-in method. The built-in and one-against-one methods do not depend on the threshold value, of course. They are shown on the same graph for comparison purpose. However, it is interesting to notice that by increasing the threshold the prediction accuracy of DCSVM on \emph{letter} data sets decreases from a comparable value with one-versus-one method's accuracy (which performs best on this data set) to the accuracy of the built-in method. With a threshold value $\theta = 2\%$ the prediction accuracy of DCSVM is still above the accuracy of the built-in method (for the \emph{letter} data set).

%%%%%%%%%%%%%%% tables
% latex table generated in R 3.4.3 by xtable 1.8-2 package
% Mon Apr 16 10:19:30 2018
\begin{table}[ht]
\caption{Prediction accuracies for different split thresholds}\label{table:accuracies}
\centering
\begin{tabular}{|r|l|r|r|rrrr|}
  \hline
    &         &    &      &\multicolumn{4}{c|}{DCSVM}\\
 No & Dataset & BI & OvsO & $\theta=2$ & $\theta=1$ & $\theta=0.1$ & $\theta=0.01$ \\
  \hline
1 & artificial1 & 98.85 & 98.76 & 98.70 & 98.70 & 98.76 & 98.76 \\
  2 & iris & 96.97 & 96.97 & 96.97 & 96.97 & 96.97 & 96.97 \\
  3 & segmentation & 27.71 & 27.71 & 29.44 & 29.44 & 29.44 & 29.44 \\
  4 & heart & 58.82 & 58.82 & 59.12 & 59.12 & 59.12 & 59.12 \\
  5 & wine & 96.97 & 96.97 & 96.97 & 96.97 & 96.97 & 96.97 \\
  6 & wine-quality & 62.33 & 62.33 & 62.39 & 62.39 & 62.39 & 62.39 \\
  7 & glass & 87.55 & 91.70 & 91.29 & 91.29 & 91.29 & 91.29 \\
  8 & covertype & 49.95 & 49.95 & 49.95 & 49.95 & 49.95 & 49.95 \\
  9 & svmguide4 & 57.88 & 71.21 & 72.73 & 72.73 & 72.73 & 72.73 \\
  10 & vowel & 94.34 & 97.26 & 97.26 & 97.26 & 97.26 & 97.26 \\
  11 & usps & 94.17 & 93.94 & 93.89 & 94.08 & 94.17 & 94.17 \\
  12 & letter & 95.25 & 96.43 & 95.52 & 96.00 & 96.41 & 96.41 \\
  13 & poker & 55.94 & 55.96 & 55.56 & 55.83 & 55.96 & 55.96 \\
  14 & sensorless & 97.46 & 98.87 & 98.32 & 98.60 & 98.86 & 98.87 \\
   \hline
\end{tabular}
\end{table}
% latex table generated in R 3.4.3 by xtable 1.8-2 package
% Mon Apr 16 10:19:32 2018
\begin{table}[ht]
\caption{DCSVM: Average number of steps per decision, for different split thresholds}\label{table:steps}
\centering
\begin{tabular}{|r|l|rrrr|}
  \hline
 No & Dataset & $\theta=2$ & $\theta=1$ & $\theta=0.1$ & $\theta=0.01$ \\
  \hline
1 & artificial1 & 3.76 & 3.75 & 3.67 & 3.67 \\
  2 & iris & 1.71 & 1.71 & 1.71 & 1.71 \\
  3 & segmentation & 5.63 & 5.63 & 5.63 & 5.63 \\
  4 & heart & 4.00 & 4.00 & 4.00 & 4.00 \\
  5 & wine & 1.69 & 1.69 & 1.69 & 1.69 \\
  6 & wine-quality & 4.85 & 4.86 & 4.87 & 4.87 \\
  7 & glass & 4.09 & 4.12 & 4.12 & 4.12 \\
  8 & covertype & 5.93 & 5.93 & 5.93 & 5.93 \\
  9 & svmguide4 & 4.63 & 4.88 & 4.88 & 4.88 \\
  10 & vowel & 5.41 & 5.41 & 5.41 & 5.41 \\
  11 & usps & 7.16 & 7.29 & 7.80 & 7.80 \\
  12 & letter & 17.63 & 19.56 & 22.29 & 22.29 \\
  13 & poker & 8.36 & 8.40 & 8.43 & 8.43 \\
  14 & sensorless & 5.44 & 5.49 & 6.28 & 6.93 \\
   \hline
\end{tabular}
\end{table}
% latex table generated in R 3.4.3 by xtable 1.8-2 package
% Mon Apr 16 10:19:33 2018
\begin{table}[ht]
\caption{DCSVM: Average support vectors per decision, for different split thresholds}\label{table:svs}
\centering
\begin{tabular}{|r|l|rrrr|}
  \hline
No & Dataset & $\theta=2$ & $\theta=1$ & $\theta=0.1$ & $\theta=0.01$ \\
  \hline
1 & artificial1 & 115.17 & 117.29 & 127.22 & 127.22 \\
  2 & iris & 32.47 & 32.47 & 32.47 & 32.47 \\
  3 & segmentation & 305.49 & 305.49 & 305.49 & 305.49 \\
  4 & heart & 270.18 & 270.18 & 270.18 & 270.18 \\
  5 & wine & 66.41 & 66.41 & 66.41 & 66.41 \\
  6 & wine-quality & 1154.49 & 1155.55 & 1157.13 & 1157.13 \\
  7 & glass & 112.14 & 114.37 & 114.37 & 114.37 \\
  8 & covertype & 4528.47 & 4528.47 & 4528.47 & 4528.47 \\
  9 & svmguide4 & 236.58 & 245.71 & 245.71 & 245.71 \\
  10 & vowel & 218.36 & 218.36 & 218.36 & 218.36 \\
  11 & usps & 785.21 & 798.41 & 846.84 & 846.84 \\
  12 & letter & 2822.54 & 3110.42 & 3307.19 & 3307.19 \\
  13 & poker & 22166.49 & 22735.92 & 22803.01 & 22807.62 \\
  14 & sensorless & 2977.89 & 3057.77 & 3404.91 & 3735.30 \\
   \hline
\end{tabular}
\end{table}
Table~\ref{table:accuracies} shows side-by-side accuracies of multi-class classification using (i) built-in (BI), (ii) one-against-one (OvsO), and (iii) DCSVM (for a few threshold values $\theta$). DCSVM performs very well in terms of accuracy (compared to the other methods) for all data sets, for threshold values $\theta \in \{ 2\%, 1\%,0.1\%,0.01\%\}$ (the larger the threshold, the better the accuracy, in general). A larger threshold $\theta$ may increase the prediction speed (Table~\ref{table:steps}) and reduce the computation effort (Table~\ref{table:svs}). Interesting to notice: Table~\ref{table:steps} shows that in all cases displayed in the table the number of decision steps is less than $k-1$, where $k$ is the number of classes in the respective data set. DCSVM outperforms (even for very small threshold) one-against-one and its improvement DAGSVM \cite{platt:99:dags}, which reaches multi-class prediction after $k-1$ steps.

%\begin{figure}[ht]
%\centering
%\includegraphics[scale=0.35]{stepsvsthshold.png}
%\caption{Average prediction steps for different thresholds}\label{fig:stepsvsthshold}
%\end{figure}

\begin{figure}[ht]
\centering
\includegraphics[scale=0.4]{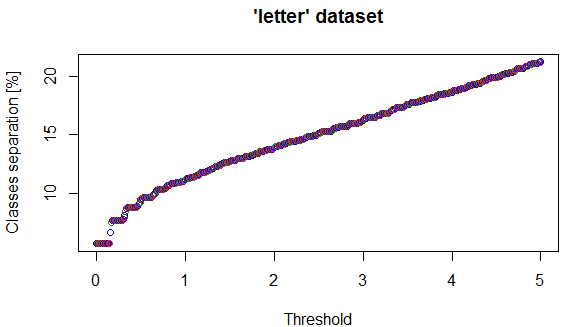}
\includegraphics[scale=0.4]{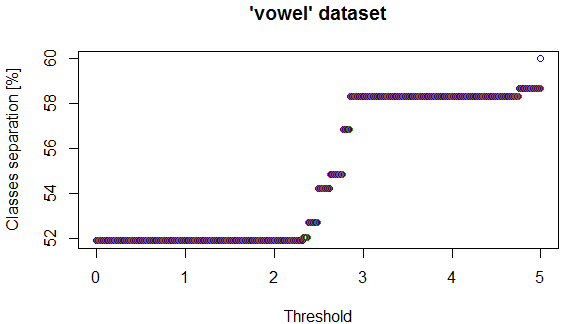}\\[3mm]
\includegraphics[scale=0.4]{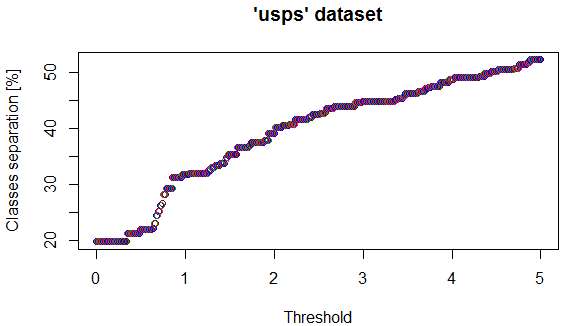}
\includegraphics[scale=0.4]{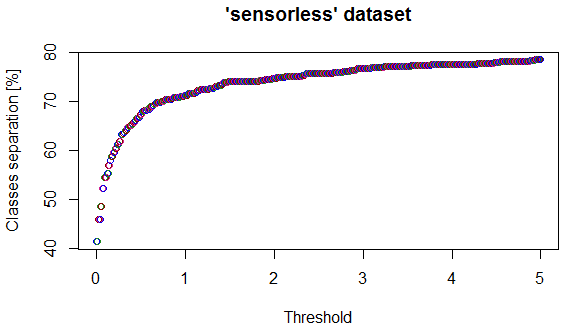}
\caption{Number of separated classes for different thresholds}\label{fig:sepclsvsthshold}
\end{figure}

The \emph{All-Predictions} table in \Cref{fig:allpredict} collects all information used by DCSVM to construct its multi-class prediction strategy (the \emph{dcsvm} decision tree in Algorithm~\ref{alg:dcsvm}). The same information can be used to predict how much separation can be achieved for different threshold values. For instance, for the \emph{glass} data set \emph{All-Predictions} table in \Cref{fig:allpredict} and for a threshold value $\theta = 2\%$ there are 58 entries in the table where the percentage of predicting one class or the other is at least $100 - \theta  = 98\%$ (out of a total of $15\times6 = 80$ entries in the table). The percentage $58/80=72.5\%$ is a good indicator of purity for DCSVM with threshold $\theta = 2\%$: the higher the percentage, the more separation is produced at each step and hence a shallow decision tree.
\Cref{fig:sepclsvsthshold} shows the class separation percentages for threshold values $0\le \theta \le 5$ and four data sets (\emph{letter}, \emph{vowel}, \emph{usps}, and \emph{sensorless}). Intuitively, as threshold increases so does the separation percentage. The \emph{letter} and \emph{usps} data sets display an almost linear increase of separation with threshold. \emph{sensorless} displays a sharp increase for small threshold values, then it tends to flatten, that is, not much gain for significant increase in threshold (and hence possibly less accuracy). Lastly, \emph{vowel} displays a step-like behavior: not much gain in separation until threshold value reaches approx $\theta = 2.3\%$, a steep increase until $\theta$ approaches $3\%$, then nothing much happens again. One can use these indicators to decide the trade-off between speed and accuracy of predictions.

\section{Conclusion}\label{sec:conclusion}
In this paper we present DCSVM, a fast algorithm for multi-class classification using Support Vector Machines. Our method relies on dividing the whole training data set into two partitions that are easily separable by a single binary classifier. Then, a prediction between the two training set partitions would eliminate one or more classes at the time. The algorithm continues recursively until a final binary decision is made between the last two classes left in the competition. Our algorithm performs consistently better than the existent methods on average. In the best case scenario, our algorithm makes a final decision between $k$ classes in $O(\log k)$ decision steps between different partitions of the training data set. In the worst case scenario, DCSVM makes a final decision in $k-1$ steps, which is not worse than the existent techniques.

The SVM divide and conquer technique we present for multi-class classification can be easily used with any binary classifier. It is rather a consequence of increasing data sparsity with the dimensionality of the space, which can be exploited, in general, in favor of producing fast multi-class classification using binary classifiers. Our experimental results on a few popular data sets show the applicability of the method.

%\appendix
%\section{An example appendix}
%\lipsum[71]
%
%\begin{lemma}
%Test Lemma.
%\end{lemma}

%\section*{Acknowledgments}
%We would like to acknowledge the assistance of volunteers in putting
%together this example manuscript and supplement.

%\bibliographystyle{siamplain}
\bibliographystyle{plain}
%\bibliography{references}

\end{document}